\documentclass[10pt]{article}
\usepackage[preprint]{tmlr}
\usepackage{amsmath,amssymb,amsthm,microtype}
\usepackage[hidelinks]{hyperref}
\usepackage{url}
\usepackage{tikz}
\usetikzlibrary{arrows.meta}
\definecolor{figink}{HTML}{243549}
\definecolor{figteal}{HTML}{137C80}
\definecolor{figblue}{HTML}{416CB0}
\definecolor{figrust}{HTML}{B65F3F}
\definecolor{figgray}{HTML}{7B8490}
\definecolor{figpale}{HTML}{F3F6F9}
\newtheorem{theorem}{Theorem}

\newtheorem{corollary}[theorem]{Corollary}
\newcommand{\R}{\mathbb R}

\newcommand{\norm}[1]{\left\lVert#1\right\rVert}
\newcommand{\minus}{\setminus}
\newcommand{\cur}{\mathcal C}
\newcommand{\bits}{\mathsf B}
\title{Exact Unlearning After Data Curation\thanks{OpenAI Codex assisted with prose writing.}}
\author{Priyank Jayraj, Devanshu Sahoo, and Uday Mittla\\
\textnormal{Birla Institute of Technology and Science, Pilani}\\
\textnormal{Pilani, Rajasthan, India}}
\date{}
\hypersetup{pdfauthor={Priyank Jayraj, Devanshu Sahoo, Uday Mittla},pdftitle={Exact Unlearning After Data Curation}}
\begin{document}
\maketitle
\begin{abstract}
Training pipelines often remove similar examples before fitting a model.
Deleting a raw record can change that selection and admit an example that
was excluded before. What must a service retain to reproduce training
after such requests? We answer this question for a fixed rule in which
earlier similar records block later ones. A limit on the total number of
future deletions identifies exactly which records can still enter training.
We use this boundary to build a summary of training statistics that can
be repaired exactly after repeated deletions. The repaired summary is the
same as rebuilding it on the surviving data with the remaining budget.
We also construct finite-precision features for ridge regression, a linear
learner with a penalty on coefficient size. With exact rational outputs,
allowing one more deletion changes the required private information.
It rises from zero to every private feature sign in the excluded candidates. An encoding
attains this conditional private-information minimum on the constructed
family throughout a sequence of requests.
On 1,001 AskUbuntu development records, deleting an initially excluded
record admits another record at 13 of 51 fixed similarity cutoffs. Exact
membership repair passes 13,056 deletion-step checks. Two model-based
semantic calibrations fail their fixed criteria, so the downstream task
study is not run. The central result is a retention contract: the future
deletions a service promises to support determine which excluded data
can still be needed.
\end{abstract}

\section{Introduction}
Training begins before a model is fitted. A pipeline may first remove
near-duplicate documents, choose representative examples, or filter a
collection for quality. These decisions depend on relationships among
records. When a record is later withdrawn, applying the same selection
rule to the surviving collection can change which examples enter training.
A deletion service that updates only the old training set can therefore
reproduce the wrong model.

Semantic deduplication makes this issue concrete. It compares document
representations and removes examples that are too similar to earlier
ones \citep{semdedup}. Consider the three-record illustration in
Figure~\ref{fig:chain}. Record $a$ precedes $b$, which precedes $c$.
Under the rule studied here, every earlier raw neighbor blocks a record,
even when that neighbor is itself excluded. The illustration shows why
this makes an excluded record relevant to a later model fit.

\begin{figure}[t]
\centering
\begin{tikzpicture}[x=1cm,y=1cm,
 font=\sffamily\fontsize{9}{11}\selectfont,text=figink,
 record/.style={rounded corners=3pt,minimum width=.86cm,minimum height=.72cm,
   draw=figgray!45,fill=white,line width=.65pt,font=\sffamily\bfseries\normalsize},
 chosen/.style={record,draw=figteal,fill=figteal!10,text=figteal},
 blocked/.style={record,draw=figgray!45,fill=white,text=figgray},
 link/.style={-{Stealth[length=4pt,width=4pt]},line width=.85pt,draw=figgray!70}]
\path[use as bounding box] (0,0) rectangle (15.8,3.52);
\fill[figpale,rounded corners=5pt] (0,.64) rectangle (6.32,3.52);
\fill[figteal!5,rounded corners=5pt] (9.48,.64) rectangle (15.8,3.52);
\node[anchor=west,font=\sffamily\bfseries\small] at (.28,3.14) {Before deletion};
\node[anchor=east,text=figgray] at (6.04,3.14) {Priority: $a\prec b\prec c$};
\node[chosen] (a0) at (1.16,2.24) {$a$};
\node[blocked] (b0) at (3.16,2.24) {$b$};
\node[blocked] (c0) at (5.16,2.24) {$c$};
\draw[link] (a0) -- node[above,font=\sffamily\scriptsize] {blocks} (b0);
\draw[link] (b0) -- node[above,font=\sffamily\scriptsize] {blocks} (c0);
\node[text=figteal] at (1.16,1.58) {selected};
\node[text=figgray] at (3.16,1.58) {excluded};
\node[text=figgray] at (5.16,1.58) {excluded};
\node[font=\sffamily\bfseries\small] at (3.16,.98) {Training set: $\{a\}$};
\draw[-{Stealth[length=6pt,width=6pt]},line width=1.2pt,figrust]
 (6.63,2.24) -- (9.16,2.24);
\node[align=center,text=figrust,font=\sffamily\bfseries\small] at (7.9,2.9) {Delete $b$};
\node[align=center,text=figgray,font=\sffamily\scriptsize] at (7.9,1.54) {It never entered\\the original fit.};
\node[anchor=west,font=\sffamily\bfseries\small] at (9.76,3.14) {Select again from survivors};
\node[chosen] (a1) at (10.64,2.24) {$a$};
\node[record,draw=figrust!60,dashed,text=figrust,fill=white] (b1) at (12.64,2.24) {$b$};
\draw[figrust,line width=.85pt] (12.32,1.99) -- (12.96,2.49);
\node[chosen] (c1) at (14.64,2.24) {$c$};
\node[text=figteal] at (10.64,1.58) {selected};
\node[text=figrust] at (12.64,1.58) {deleted};
\node[text=figteal,font=\sffamily\bfseries\small] at (14.64,1.58) {newly selected};
\node[font=\sffamily\bfseries\small] at (12.64,.98) {Training set: $\{a,c\}$};
\draw[figgray!30,line width=.5pt] (.28,.32) -- (15.52,.32);
\node[font=\sffamily\small] at (7.9,-.02) {Updating only the old training set would keep $\{a\}$ and miss $c$.};
\end{tikzpicture}
\caption{An excluded record can still control selection. Arrows show blocking
by earlier raw neighbors, including excluded neighbors. There is no $a,c$
edge. Deleting $b$ changes the training set and can change the fitted model.}
\label{fig:chain}
\end{figure}

Let $\mathcal D$ be the raw collection, $F$ the records deleted so far,
$\cur$ the selection rule, and $A$ the learner. We study the target
\begin{equation}\label{eq:target}
                 A\bigl(\cur(\mathcal D\minus F)\bigr).
\end{equation}
It means selecting from the surviving raw records and then training again.
Updating the original selected set instead gives
$A(\cur(\mathcal D)\minus F)$. These expressions place deletion on
different sides of selection, so they define different retention targets.

\paragraph{The practical question: what must remain recoverable?}
Keeping the entire raw collection allows fresh reconstruction after every
request. A service may instead keep summaries and discard raw payloads,
such as feature rows and labels. It then needs to know which excluded
records can affect a future target. We make the permitted future requests
explicit through a cumulative deletion budget: the maximum number of
records that may be deleted over the life of the retained state.

For the rule above, an excluded record becomes selected only after all
its earlier blockers are deleted. If it has more surviving blockers than
remaining deletions, it cannot enter training within the contract. Its
payload can be discarded for this target. If the budget is large enough
to remove those blockers, the record remains potentially relevant.
The budget thus connects an operational promise to a concrete retention
decision. It decreases with every newly deleted record. A fresh allowance after
each request would be a different promise.

\paragraph{Two connected results.}
First, we derive an exact summary of additive training statistics for this
selection rule. Its update removes deleted contributions and adds newly
selected ones. More strongly, the updated summary equals fresh construction
on the surviving records at the remaining budget. This permits repeated
requests without accumulating a different logical state.

Second, we show that the budget can determine necessary information, not
just sufficient storage. We construct features that are exactly stored
in single precision (FP32) and used by both selection and ridge regression.
Under one budget, every model is zero. Increasing the budget by one makes
every private feature sign in the excluded candidates necessary for exact
model recovery across the allowed requests. A sequential encoding retains
exactly the signs that can still matter. The current model alone cannot
describe this future information requirement.

The natural-data study examines the same selection and repair mechanism on
fixed AskUbuntu score graphs. It checks exact membership and continuing
state across all original cutoffs. Its separate semantic calibration uses
two local language models and zero human pair review. Both calibration
versions fail their fixed selection criteria; the threshold-qualified
utility and cost study therefore remains unexecuted. Section~\ref{sec:evidence}
presents the completed evidence and its empirical scope together.

\paragraph{Position of the contribution.}
Summaries for preprocessing-aware unlearning \citep{cao}, incremental view
maintenance \citep{ivmunlearning}, and state that supports later deletions
\citep{ginart,cohencontrol} provide the starting point. We specialize signed
provenance \citep{green,aggregate} to a fixed ordered selection rule and
prove its remaining-budget state identity. The information result joins
finite stored features, stable similarity decisions, exact learner outputs,
and a matching sequential state in one construction. The resulting design
question is precise: which future requests can this retained state still
answer? Section~\ref{sec:related} compares the relevant contracts.

\section{From a deletion promise to a retention rule}\label{sec:setup}
We first define the selection rule and the service contract. The key
objects are the blockers of each record and the deletions still allowed.
No model-specific algebra is needed to determine future eligibility.

\subsection{A fixed ordered selection rule}
Each raw record $v\in\mathcal D$ has a stable identifier, a fixed feature
vector $z_v$, and, when needed, a label $y_v$. A graph connects pairs that
qualify as similar. A fixed total order $\prec$ gives one record priority
over another. An edge can require both membership in the same fixed
partition and cosine similarity strictly above a threshold $\tau$.
Define the \emph{blocker set} of $v$ as its earlier raw neighbors:
\[
                  B(v)=\{u:u\prec v,\ \{u,v\}\text{ is an edge}\}.
\]
After a cumulative deletion set $F$, selection returns
\begin{equation}\label{eq:activation}
 S_F=\cur(\mathcal D\minus F)=\{v\notin F:B(v)\subseteq F\}.
\end{equation}
A record must survive, and every one of its blockers must be gone.
In Figure~\ref{fig:chain}, the blocker sets are
$B(a)=\varnothing$, $B(b)=\{a\}$, and $B(c)=\{b\}$.

This is \emph{raw-neighbor suppression}. Greedy selection against only
previously retained neighbors is a different rule: it would select both
$a$ and $c$ in the initial chain. For nonnegative thresholds, the
upper-triangle column-maximum test in the released SemDeDup code has
our suppression form \citep{semcode}. We fix this operator's graph and
order, including any pair-eligibility partition.

\paragraph{What stays fixed under deletion.}
Features, pair eligibility, relative priority, and threshold remain
unchanged on surviving records. We call this \emph{restriction consistency}.
It makes the retained graph the original graph restricted to surviving
vertices. Ordering records by a hash of each stable ID, with an ID tie
break, preserves their relative priority. Repeating a corpus-dependent
shuffle with the same seed need not. Similarly, refitting clusters or a
text encoder can change surviving edges and defines a different target.
Here the pretrained representation is held fixed throughout repair.

\subsection{Which records can still enter training?}
Let $K$ be the original cumulative record-deletion budget. After deleting
$F$, the remaining budget is $h=K-|F|$. We say a surviving record is
\emph{eligible} if some continuation using at most $h$ further deletions
can select it. The eligible set is exactly
\begin{equation}\label{eq:eligible}
 E(F,h)=\{v\notin F:|B(v)\minus F|\le h\}.
\end{equation}
To select $v$, a request must remove all its surviving blockers. This is
possible precisely when their number fits in the remaining budget.
Eligibility includes currently selected records, whose remaining blocker
count is zero, as well as some currently excluded records.

The rule supports permanent discard within the contract. Suppose a record
has $r>h$ remaining blockers. After a further $\ell$ deletions, it has at
least $r-\ell$ blockers and the budget becomes $h-\ell$. It still has
more blockers than the budget allows. Thus eligibility can only shrink
along a valid sequence. The same argument shows that a currently selected
record remains selected until it is itself deleted.

For example, consider a surviving record with four blockers and a budget
of three further deletions. Removing two of its blockers leaves two
blockers and one deletion available. Removing an unrelated record instead
leaves all four blockers and a budget of two. Neither action can make the
record reachable. This simple count is the basis for safe payload discard;
it does not require anticipating the order of later requests.

\subsection{What the service returns and what it may read}
Three recovery targets are useful. \emph{Corpus recovery} returns the
selected identifiers $S_F$. \emph{Model recovery} returns the specified
learner fitted on those records. \emph{State recovery} returns a designated
summary that supports the remaining requests. We compare that summary
with a fresh construction on the survivors using the same remaining
budget. Two states can give the same current model while supporting
different future requests, as Section~\ref{sec:memory} will demonstrate.

For the storage question, all usable private-data-dependent information
is charged to the retained state. This includes models, caches, tickets,
remote stores, retained random seeds, and earlier outputs that repair can
read. There is no free reread of surviving candidate features. Public
metadata is fixed independently of the private payload, but its physical
storage still counts in a complete implementation. These conventions
separate a private-information bound from a total-byte comparison.

The state identity proved next concerns a logical coefficient map.
It is an exact specification for continuing computation. Physical erasure,
previous releases, and external privacy obligations are discussed in
Appendix~\ref{sec:limits}.

\section{Building an exact repair summary}\label{sec:method}
Many learners can be fitted from additive statistics. Rather than store
one selected total, we represent how that total changes when blockers
are deleted. The representation records the conditions under which each
row contributes, including the condition that the row itself survives.

\subsection{From one record to a signed coefficient map}
Let $t(v)\in\R^s$ be a fixed statistic of record $v$, with $s$ coordinates.
It is record-local: it does not change when other records are deleted.
For each record $u$, introduce a deletion indicator $x_u$, equal to one
when $u$ is deleted and zero otherwise. For an identifier set $J$, write
$x_J=\prod_{u\in J}x_u$, with $x_\varnothing=1$.

The statistic of $v$ should contribute exactly when all its blockers are
deleted and $v$ survives. Its contribution is therefore
$t(v)(1-x_v)x_{B(v)}$. Summing over records gives
\begin{align}
 T(x)&=\sum_v t(v)(1-x_v)x_{B(v)}
       =\sum_J\alpha_J x_J,\label{eq:poly}\\
 \alpha_J&=\sum_{v:B(v)=J}t(v)
       -\sum_{v:B(v)\cup\{v\}=J}t(v).\label{eq:coeff}
\end{align}
A key $J$ names a set of deletions, and its signed value $\alpha_J$ is
added when every deletion in that key has occurred. Each record contributes
to at most two keys: a positive term for removing its blockers and a
negative term for also removing the record. Combining equal keys can
cancel contributions. This is a signed Boolean specialization of
provenance for data and aggregates \citep{green,aggregate}.

\paragraph{Worked example: deletion requires addition and cancellation.}
For the three-record chain, abbreviate the statistics as
$A_0=t(a)$, $B_0=t(b)$, and $C_0=t(c)$. Then
\begin{align*}
 T(x)&=A_0(1-x_a)+B_0(1-x_b)x_a+C_0(1-x_c)x_b\\
 &=A_0+(B_0-A_0)x_a+C_0x_b-B_0x_ax_b-C_0x_bx_c.
\end{align*}
With an original budget $K=2$, every displayed key is retained.
Deleting $b$ sets $x_b=1$ and leaves a budget of one:
\[
                  T_{\rm retained}(x_a,x_c)
                   =(A_0+C_0)-A_0x_a-C_0x_c.
\]
The $B_0$ terms cancel, while $C_0$ enters the current total. This is
exactly the expression built from the two unconnected survivors $a,c$.
Deleting a previously excluded record thus needs more than subtracting
its old training contribution, which was zero.

\subsection{Truncate once, then reduce the budget}
The degree of a key is its number of identifiers. A key with more than
$K$ identifiers can never be activated by at most $K$ deletions. We can
therefore omit coefficients of degree above $K$ at construction. The
important question is whether an omitted term could become needed after
some of its identifiers are deleted. It cannot when the remaining budget
falls by the same amount.

\begin{theorem}[Finite-horizon state repair]\label{thm:repair}
Keep only coefficients with $|J|\le K$. For every $|F|\le K$,
\[
                 \sum_{v\in S_F}t(v)=\sum_{J\subseteq F}\alpha_J.
\]
After an allowed cumulative deletion $F$, form
\begin{equation}\label{eq:update}
 \beta_L=\sum_{J:J\minus F=L}\alpha_J,
                       \qquad |L|\le K-|F|.
\end{equation}
In exact arithmetic, this is the same coefficient map as fresh construction
on $\mathcal D\minus F$ with horizon $K-|F|$. Repeated requests compose
independently of their batching and order for the same final deleted set.
\end{theorem}
\begin{proof}
A monomial equals one precisely when its key is contained in $F$.
Keys of degree above $K$ therefore contribute zero to the current total.
Substituting $x_u=1$ for every $u\in F$ removes deleted records through
their $(1-x_u)$ factor. For each surviving record it replaces $B(v)$ by
$B(v)\minus F$. Collecting the resulting keys gives Equation~\eqref{eq:update}.
Any term discarded at construction had degree at least $K+1$. Deleting
$|F|$ identifiers reduces its degree by at most $|F|$, leaving degree
at least $K+1-|F|$, above the new cutoff. It cannot affect the fresh
remaining-budget map. Substitutions commute, and the same argument applies
at each intermediate horizon.
\end{proof}

The empty key is the currently selected statistic. Other keys describe
possible future changes. The theorem preserves both parts, so a request
can be answered now without losing the state needed by the next request.
It also supplies a direct oracle: reconstruct the same map from retained
records and compare it with the updated map.

The decreasing budget is essential. In the chain with $K=1$, only the
constant and degree-one terms are stored. Deleting $b$ still gives the
correct current statistic $A_0+C_0$, but the remaining budget is zero.
If one incorrectly kept the old degree-one allowance and then deleted
$a$, that truncated expression would give $B_0+C_0$ instead of $C_0$.
The second request exceeds the original contract. Renewing a deletion
budget requires a state prepared for the larger promise.

\subsection{Indexed updates and their cost}\label{sec:implementation}
An indexed implementation, called P-I in the supplement, stores three
objects: the nonzero coefficient map, an inverse index from each identifier
to keys containing it, and buckets grouping keys by degree. A singleton
deletion follows a short sequence:
\begin{enumerate}
\item Find the keys containing the deleted identifier through the inverse index.
\item Remove that identifier from each key and combine values whose reduced keys coincide.
\item Remove zero coefficients and their index entries.
\item Decrease the remaining budget and discard the newly excessive degree bucket.
\end{enumerate}
A coefficient can move without changing its value. Vector addition is
needed only when keys merge. For example, deleting $b$ in the chain
merges the old constant with key $\{b\}$ and merges $\{a\}$ with
$\{a,b\}$. Their values become $A_0+C_0$ and $-A_0$, respectively.
The scanning implementation P-S applies the same substitution by visiting
all keys. The inverse index avoids scanning unaffected entries.

Let $M_0$ be the initial number of keys and $L_0$ their total length.
Let $Q$ count supplied integer request identifiers across the sequence.
For statistic dimension $s$, define the initial potential
$\Phi_0=\sum_J |J|(|J|+1)/2$. The indexed substitution core has
sequence-wide cost
\begin{equation}\label{eq:work}
 O(Q+\Phi_0+sM_0)\subseteq O(Q+K L_0+sM_0),
\end{equation}
under expected constant-time dictionary operations apart from tuple work.
Shrinking a length-$j$ key costs $O(j)$ and decreases its potential by $j$.
Each merge removes an entry, so at most $M_0$ vector merges occur.
Initial incidences are only removed, giving at most $L_0$
inverse-incidence removals. This amortized argument allows an individual
request to touch all keys.

The coefficient map and indices use $O(sM_0+M_0+L_0)$ words, plus
membership storage. Construction, learner solves, rational bit lengths,
and serialization are separate costs. The frozen Python service also
performs request validation, a full public-ID lookup, sorting, and per-call
work. Appendix~\ref{app:work} gives these costs and retry behavior.
Equation~\eqref{eq:work} describes the substitution core rather than the
complete service.

\subsection{Ridge regression and contributor withdrawal}
For ridge regression, let $z_v\in\R^d$ be a stored feature column and
$y_v\in\R^q$ its response column. A matrix $W\in\R^{d\times q}$
predicts $W^\top z_v$. Use the record statistic
\[
              t(v)=(z_vz_v^\top,z_vy_v^\top,1).
\]
Its selected sum is $(G,H,n)$: a feature cross-product matrix, a
feature-response matrix, and the selected count. Average-loss ridge with
fixed $\lambda>0$ solves
\begin{equation}\label{eq:ridge}
                       (G+\lambda n I)W=H.
\end{equation}
There is no intercept, all coordinates are penalized, and the empty
selected set returns zero. The factor $n$ follows from using average loss;
repairing moments while retaining the old count would change this learner.
Cosine-normalized scoring rows also do not replace the stored raw rows in
its objective.

This specialization makes corpus and model repair distinct operations.
The coefficient update supplies exact moments and count. A subsequent
solve produces the head. Dense statistics have $O(d^2+dq)$ coordinates,
and a dense solve costs $O(d^3+d^2q)$. Whether maintaining these moments
is preferable to retaining compact eligible rows depends on the complete
costs considered in Section~\ref{sec:contracts}.

The same algebra supports withdrawal of all records belonging to one
source. Assume sources partition the records, and write $g(v)$ for the
source of $v$. Define the distinct blocking sources
$U(v)=\{g(u):u\in B(v)\}$. With one deletion indicator per source,
the contribution becomes
\[
                       t(v)(1-x_{g(v)})x_{U(v)}.
\]
Repeated occurrences of one source collapse because a Boolean indicator
satisfies $x^2=x$. If $g(v)\in U(v)$, this expression is zero: removing
that blocking source would also remove $v$. Otherwise the previous
construction and proof apply with a source-count budget. Several blocking
records from one source cost only one withdrawal request. The ownership
partition must therefore be part of the target, alongside the graph.

\section{Why the future budget can require excluded data}\label{sec:memory}
The summary above is sufficient for exact repair. We now ask how much
information any exact model-recovery service can require. The example
uses a collection of excluded candidates. Each has a private word of
feature signs, a shared group of blockers, and one individual blocker.
Removing its blockers exposes that candidate to the learner.

The construction connects graph selection to model information. Public
feature coordinates determine which deletions expose each candidate.
Private coordinates determine the resulting ridge head without changing
the graph. The same stored rows serve both purposes. We first describe
the feature family, then derive its current and future state requirements.

\subsection{A graph that exposes one private word at a time}
Fix $b\ge1$, $m\ge2$, and powers of four $k,D$ with
$d=2+k+D\le2^{20}$. Public signatures
$u_i\in\{\pm1/\sqrt{k}\}^k$ satisfy $|u_i^\top u_j|\le1/6$ for
$i\ne j$; orthogonal Hadamard rows suffice when $m\le k$.
Sort fixed public IDs by their actual priority and assign roles: an anchor
$a$, $b$ common blockers $S_0$, $m$ individual blockers $p_i$, and $m$
candidates $w_i$, in that order. Use orthogonal coordinate blocks
$e_0,e_1,U,H$ of dimensions $1,1,k,D$, and set $\alpha=1/128$.
The public $U$ block identifies a candidate through $u_i$. The $H$ block
carries its private word $z_i\in\{\pm1/\sqrt D\}^D$. The words vary
independently across candidates; only their signs are private. The raw rows are
\begin{equation}\label{eq:rows}
\begin{aligned}
 a&=e_0,\\
 s_j&=\tfrac12 e_0+\tfrac78 e_1,\\
 p_i&=\tfrac12e_0+\tfrac78u_i,\\
 w_i&=\tfrac12e_1+\tfrac78u_i+\alpha z_i.
\end{aligned}
\end{equation}
Candidate labels are one; every other label is zero. All nonzero entries
are exactly stored normal FP32 values in the stated dimension range.
Candidate labels, graph, signatures, roles, and blocker payloads are public.
The entire public contract is fixed independently of the private signs.

At the stored binary64 threshold $\tau=\operatorname{float64}(0.4)$,
the frozen ordered double-precision (FP64) normalization and dot product give
\[
                         B(w_i)=S_0\cup\{p_i\}
\]
uniformly over all private signs. Ideal required edges exceed $.43$;
required nonedges are below $.372$. The scorer error is below $2^{-32}$
under the explicit arithmetic assumptions in Appendix~\ref{app:scorer}.
The operation order and arithmetic assumptions are part of the scorer
contract; Appendix~\ref{app:scorer} gives the full margin proof.

\subsection{From an exposed candidate to an information bound}
Initially only the zero-label anchor is selected. At budget $b$, no
candidate can lose its $b+1$ blockers, and every target head is zero.
At budget $K=b+1$, the request $F_i=S_0\cup\{p_i\}$ selects $a,w_i$.
Each candidate has exactly $K$ blockers, so any allowed request selecting
it must equal $F_i$; no allowed request selects two candidates.
For exact rational ridge with rational $\lambda>0$, define
\[
                \kappa=16641/16384+2\lambda.
\]
Since $a\perp w_i$, the resulting head is
\begin{equation}\label{eq:hardhead}
 \theta_i=\frac{w_i}{\kappa}
     =\frac{\tfrac12e_1+\tfrac78u_i}{\kappa}+\frac\alpha \kappa z_i.
\end{equation}
Its $H$ coordinates reveal all $D$ private signs of candidate $i$.

\begin{theorem}[Sharp private-information transition]\label{thm:bits}
On this fixed public family, the minimum deterministic exact model-query
state is zero private bits at horizon $b$ and $mD$ private bits at horizon
$b+1$. The initial state of an exact sequential service has the same lower
bound. It also holds for randomized services exactly correct almost surely
for every fixed input and query.
\end{theorem}
\begin{proof}
There are $2^{mD}$ independent sign assignments. Any two differ on some
candidate $i$ and produce different heads on its public request $F_i$.
Their states must therefore differ. Conversely, storing all signs allows
reconstruction of the indicated head on $F_i$ and zero on all other
allowed queries. At horizon $b$, always returning zero suffices. There
are finitely many inputs and requests, so a randomized zero-error service
has a common probability-one set of coins on which the deterministic
distinguishability argument applies.
\end{proof}

Here a $\bits$-bit state has at most $2^{\bits}$ distinguishable values
conditional on public information and independent public coins. Usable
length or framing choices are charged as well. Each isolating request $F_i$ exposes
one candidate, but the initial state must support every candidate's request.
That is why the lower bound sums all $mD$ signs rather than the $D$ signs
revealed by a single answer. The requests delete only public blocker
payloads, so receiving those deleted payloads supplies none of the missing
private information. Public metadata, padding, and workspace consume
additional physical storage under the contract of Section~\ref{sec:setup}.

\subsection{The minimum information after earlier deletions}
For cumulative deletion $F$ under the original $K=b+1$, define
\begin{align}
 I(F)&=\{i:w_i\notin F,\ |B(w_i)\minus F|\le K-|F|\}\nonumber\\
     &=\{i:F\subseteq S_0\cup\{p_i\}\},\qquad e(F)=|I(F)|.\label{eq:live}
\end{align}
The equality uses the specific construction, not arbitrary graphs.

\begin{corollary}[Conditional retained state]\label{cor:live}
For a fixed admissible $F$, correctness for every private assignment and
every allowed continuation requires and is achievable with exactly
$e(F)D$ private bits.
\end{corollary}
\begin{proof}
Each $i\in I(F)$ can still be isolated by deleting its remaining blockers,
so its signs must be distinguishable. Storing just those words suffices.
Under further deletion $F'\supseteq F$, $I(F')\subseteq I(F)$, so a
discarded word never becomes relevant again within this horizon.
\end{proof}

A small instance illustrates why current output and future state differ.
Take one common blocker ($b=1$), two candidates ($m=2$), and four private
signs per candidate ($D=4$). Use $k=4$ with public signatures
$(1,1,1,1)/2$ and $(1,1,-1,-1)/2$, so the full feature dimension is ten.
The original budget is $K=2$. Figure~\ref{fig:live} compares three
alternative first requests and the private words they leave reachable.

\begin{figure}[t]
\centering
\begin{tikzpicture}[x=1cm,y=1cm,
 font=\sffamily\fontsize{9}{11}\selectfont,text=figink]
\path[use as bounding box] (0,0) rectangle (15.8,4.1);
\node[font=\sffamily\bfseries\small] at (7.9,3.86) {One starting state. Three alternative first deletions.};
\node[text=figgray] at (7.9,3.46) {Each leaves the current model at zero and one deletion still available.};
\foreach \x/\request/\count/\tone in {0/s/8/figteal,5.4/p_1/4/figblue,10.8/a/0/figgray}{
  \fill[\tone!5,rounded corners=5pt] (\x,0) rectangle +(5,3.06);
  \draw[\tone,line width=2pt] (\x+.2,3.06) -- (\x+4.8,3.06);
  \node[anchor=west,font=\sffamily\bfseries\small] at (\x+.27,2.66) {Delete $\request$};
  \node[text=\tone,font=\sffamily\bfseries\Large,anchor=east] at (\x+4.7,2.66) {\count};
  \node[text=\tone,anchor=east,font=\sffamily\scriptsize] at (\x+4.7,2.24) {private bits};
  \node[anchor=west] at (\x+.27,1.64) {Candidate 1};
  \node[anchor=west] at (\x+.27,.93) {Candidate 2};
}
\foreach \x/\y/\tone in {0/1.64/figteal,0/.93/figteal,5.4/1.64/figblue}{
  \foreach \j in {0,1,2,3}{
    \filldraw[fill=\tone!17,draw=\tone,line width=.55pt,rounded corners=1pt]
      (\x+2.55+\j*.49,\y-.18) rectangle +(.36,.36);
  }
}
\foreach \x/\y in {5.4/.93,10.8/1.64,10.8/.93}{
  \foreach \j in {0,1,2,3}{
    \draw[draw=figgray!35,fill=white,rounded corners=1pt,line width=.5pt]
      (\x+2.55+\j*.49,\y-.18) rectangle +(.36,.36);
  }
  \draw[figgray,line width=.7pt] (\x+2.49,\y-.25) -- (\x+4.44,\y+.25);
}
\node[text=figteal,font=\sffamily\bfseries\small] at (2.5,.32) {Both words remain reachable};
\node[text=figblue,font=\sffamily\bfseries\small] at (7.9,.32) {Only word 1 remains reachable};
\node[text=figgray,font=\sffamily\bfseries\small] at (13.3,.32) {Neither word remains reachable};
\end{tikzpicture}
\caption{Identical current models, different retention needs. In this
constructed $K=2$ instance, each candidate has four private sign bits, shown as
filled squares. Crossed words cannot affect any permitted continuation. Counts give
the conditional minimum private information and exclude public metadata and
physical overhead.}
\label{fig:live}
\end{figure}
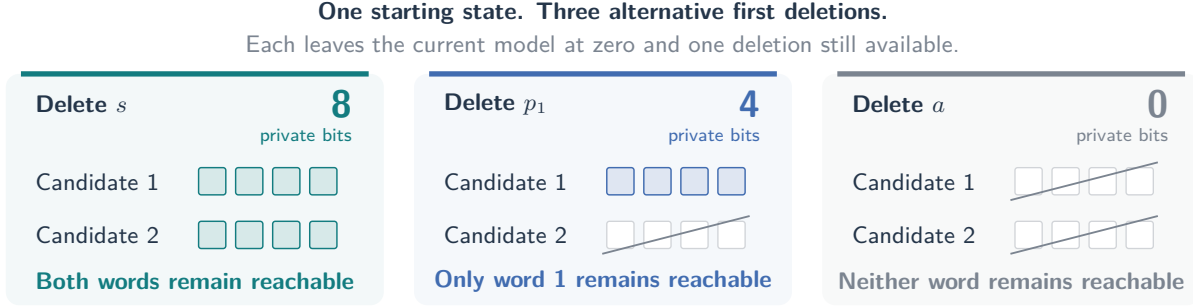

The shared blocker $s$ and individual blocker $p_1$ open different
continuations. Deleting the anchor $a$ uses budget without removing either
candidate's blockers. If the complete request is $\{s,p_1\}$, candidate 1
is already selected. Its four signs remain necessary even at remaining
budget zero, because the current head must still be recoverable.

The finite-family codec stores eligible sign words in candidate order.
A transition copies only words that remain eligible, using immutable
canonical exports. The export has a 48-byte header, a charged public
deletion ledger, and $\lceil e(F)D/8\rceil$ private payload bytes.
Fresh retained-data initialization and sequential export yield identical
bytes under the same public universe and remaining budget. The conditional
bound fixes $F$; Appendix~\ref{sec:limits} discusses request policies that
reveal information through their own adaptive choices.

\subsection{Approximate heads and finite output precision}

With expected squared parameter error at most $\varepsilon$ on each fixed
input and hard probe, the same private words yield a quantitative bound. Write $h_2$ for binary
entropy in bits and define the squared-error information function
\[
 R_{\rm sq}(v)=
 \begin{cases}
 1-h_2((1-\sqrt{1-4v})/2),&0\le v\le1/4,\\
 0,&v>1/4.
 \end{cases}
\]
A standard information-distortion argument \citep{shannon} gives
\begin{equation}\label{eq:rate}
 \bits\ge mD R_{\rm sq}\!\left(\frac{\kappa^2\varepsilon}{4\alpha^2}\right),
\end{equation}
The normalization relates head error to uncertainty about a private sign.
For a hidden sign encoded as a bit $Y_{ij}\in\{0,1\}$, the corresponding
exact head coordinate is $\alpha(2Y_{ij}-1)/(\kappa\sqrt D)$.
Rescaling an estimated coordinate to the bit interval multiplies its
squared error by $\kappa^2D/(4\alpha^2)$. The entropy bound then charges
the uncertainty removed across all $mD$ bits. Appendix~\ref{app:information}
gives the posterior-variance argument and its concavity step.

At zero error, Equation~\eqref{eq:rate} recovers the exact $mD$ lower
bound. Returning only the public part of Equation~\eqref{eq:hardhead}
uses no private signs and attains squared error $\alpha^2/\kappa^2$.
At that error the lower bound becomes zero. Thus the parameter-error
allowance controls how much of the private word must be retained.
The criterion applies to every fixed input and isolating request;
averaging over all deletion sets would dilute the difficult cases among
requests whose answer is zero.

The exact head is rational even though the input features are stored in
FP32. In the ten-dimensional example with $\lambda=1$, a hidden feature
coordinate is $\pm1/256$ and the corresponding head coordinate is
$\pm64/49409$, which has no finite binary expansion. A floating release
therefore needs a stated approximation contract. An actual certified
parameter radius below $\alpha/(\kappa\sqrt D)$ preserves every sign.
Appendix~\ref{app:certificate} gives a sufficient finite-output radius
and explains the distinction between parameter error and prediction risk.

\section{What the contract means for a training service}\label{sec:contracts}
The service must choose what it can reconstruct, then how to store it.

\subsection{Retain payloads and dependencies for different purposes}
Deleting an excluded record involves both payloads and dependencies.
Its own training contribution may be zero, but its identifier can still
appear in conditions controlling other records. Payload discard and
dependency discard serve different purposes. A correct implementation
must remove the record's selection influence as well as its own
contribution. In the coefficient representation, the substitution does
both operations and allows their effects to cancel where appropriate.

\subsection{Compare with the strongest simple storage alternative}
A natural baseline stores compact feature and label rows only for records
that remain eligible. It also retains their blocker incidences and degree
buckets. After a request, it identifies selected-set changes and updates
moments in a batch. The supplement calls this baseline B-E. It uses the
same eligibility boundary as the coefficient method, but delays expanding
a feature row into an outer product.

This is a consequential choice for high-dimensional features. A row uses
$O(d+q)$ values, while its dense ridge statistic uses $O(d^2+dq)$.
Combining many rows into a shared coefficient can amortize the larger
statistic. If few records share a key, compact rows can be preferable.
The earlier lexical development comparison favored the compact eligible-row
baseline in total bytes. A bound on necessary private information does
not rank these general implementations by physical memory.

A useful evaluation therefore compares full lifecycles under identical
requests. It includes initialization, graph and index storage, payloads,
moments, temporary gathered rows, solves, and serialization. Incremental
update time and total service time answer different questions. Fresh
reconstruction remains the correctness oracle; an optimized eligible-row
baseline is the practical cost comparator. Appendix~\ref{sec:design}
retains the prepared evaluation design and all unexecuted comparisons.

\section{Evidence on natural score graphs}\label{sec:evidence}
We examine two questions: does the fixed selection rule admit records
after deletion on natural features, and does a chosen similarity threshold
support the intended semantic application? These questions use different
oracles. Exact membership checks test the selection and repair algebra.
Model agreement tests the frozen semantic qualification criterion.

\subsection{Data, representations, and study populations}\label{sec:study}
The prepared AskUbuntu population contains 140,250 questions from the
corrected April 2024 export \citep{askubuntu}. Calibration uses 20,996
questions. Development is reserved by whole contributor source before
any effect measurement. The structural and repair diagnostics use the
original pre-guard prefix of 1,001 development-training records. They
were separately fixed after the failed calibration and retain all 51
original cutoffs: $.50,.51,\ldots,.99,.995$. No operating cutoff is
chosen from these diagnostics.

The representation is the pinned multilingual E5 base encoder \citep{e5},
with 768 stored FP32 coordinates per question. We encode complete documents
in successive 256-content-token chunks. Each chunk receives a
\texttt{query:} prefix and an attended-token mean, which is normalized.
A content-length-weighted average of the chunk vectors is normalized
again to produce the stored row. The scorer normalizes stored rows in
FP64 before cosine comparison. This fixed full-document recipe determines
the score distribution; Appendix~\ref{app:embedding} specifies its precision
and operation order.

Native accounts supply contributor proxies. Unknown and community owners
remain explicit singletons and are excluded from genuine source-withdrawal
requests. The complete split, ownership counts, source reservation, and
prepared tag-prediction task appear in Appendix~\ref{app:protocol}.
Neither diagnostic reads task labels, held-out feature rows, or model
judgments. Both reuse the same feature cache, priority order, and strict
comparison $\mathrm{score}>\tau$.

\subsection{Where deletion changes the selected set}
The grid contains markedly different graph regimes
(Table~\ref{tab:structure}). The first 25 cutoffs, $.50$ through $.74$,
produce complete graphs. At $.995$ the graph is empty. Between these
extremes, deleting an initially excluded record admits another record at
13 cutoffs, $.84$ through $.96$.

\begin{table}[t]
\centering
\small
\caption{Illustrative points from the full 51-cut grid on 1,001 development
records. The last column counts initially excluded records whose singleton
deletion admits at least one other record. These are fixed score graphs,
not selected or semantically qualified operating thresholds.}
\label{tab:structure}
\begin{tabular}{rrrrr}
\hline
Cutoff & Edges & Selected & Excluded & Excluded deletions with admission\\
\hline
.50 & 500,500 & 1 & 1,000 & 0\\
.84 & 245,841 & 7 & 994 & 1\\
.90 & 15,678 & 504 & 497 & 50\\
.96 & 197 & 952 & 49 & 7\\
.995 & 0 & 1,001 & 0 & 0\\
\hline
\end{tabular}
\end{table}

The intermediate graphs expose the distinction in Figure~\ref{fig:chain}.
At the illustrative $.90$ cutoff, 50 of the 497 excluded records each
have a deletion that admits another record. In a complete graph, every
later record has all earlier records as blockers. Deleting one excluded
record cannot remove that entire earlier set. In an empty graph, all
records are already selected. Graph density alone therefore does not
predict the excluded-blocker effect; the arrangement of blockers matters.

For uniform deletion of 128 records, the exact expected number of
admissions ranges from zero to 12.51 across the full grid. These
expectations concern the fixed panel and sampling rule. The execution
reproduced every graph with two tile sizes. A separate aggregate audit
checked all 204 exact expectations for budgets $1,8,32,128$.

\subsection{Exact repair throughout deletion sequences}
We apply the unchanged indexed algorithm to integer membership coordinates:
each record contributes its own coordinate when selected. This makes the
sum an exact indicator vector of the selected set and exercises the
coefficient mechanism without fitting a task learner.

One fixed record-deletion stream and one fixed contributor-proxy source
stream are reused at every cutoff. Each has an original horizon of 128
requests. The record stream removes 128 records; the source stream removes
261. The streams check both request units, rather than comparing their
effect sizes at matched removed-record volume.

\begin{table}[t]
\centering
\small
\caption{Exact integer repair checks over all 51 cutoffs. Counts come from
the same two fixed trajectories, not independent statistical replicates.
State equality compares logical coefficient maps and indices at the
remaining budget.}
\label{tab:repair}
\begin{tabular}{lr}
\hline
Check & Matches / checks\\
\hline
Selected membership after each singleton request & 13,056 / 13,056\\
Sequential versus batched retained state & 408 / 408\\
Sequential versus reverse-order retained state & 408 / 408\\
Retained state versus fresh reconstruction & 408 / 408\\
\hline
\end{tabular}
\end{table}

Every singleton update matches a direct raw-blocker membership oracle
(Table~\ref{tab:repair}). After $1,8,32,128$ requests in each trajectory,
we also compare sequential, batched, and reverse-order processing.
Recomputing the retained-data graph and constructing the summary at the
remaining horizon gives the same logical state at all 408 checkpoints.
These comparisons test the stronger conclusion of Theorem~\ref{thm:repair},
not just equality of the current selected count. Private-state comparisons
were executed in the cluster run; an independent aggregate audit checked
the bound receipts and reported counts.

\subsection{Semantic qualification: two failed calibrations}\label{sec:judges}
The intended semantic application requires a qualified similarity cutoff.
Annotation uses pinned Qwen2.5-7B-Instruct \citep{qwen} and
Mistral-Nemo-Instruct-2407 \citep{mistral}, with no human pair review.
Both receive the complete pair texts under a fixed rubric, without IDs,
scores, tags, or each other's answers. A pair is positive only when both
valid outputs call it an exact copy or substantially the same meaning.
Disagreements, other labels, uncertainty, and invalid or unavailable
outputs contribute zero positive mass and remain in the denominator.
The target is this joint model-label function. Model agreement and
order sensitivity are assessed as such \citep{judge}.

Version 3 completed 90 canonical pairs and nine swapped diagnostics per
model. It produced zero joint positives and no qualifying cutoff.
Version 4 is a separately versioned follow-up, with the earlier negative
result preserved. It enumerates all 220,405,510 unordered calibration
pairs and samples up to 30 pairs without replacement from each of 13
score strata. The realized sample has 360 canonical pairs and 36 swapped
diagnostics per model. Inverse-inclusion weights estimate the joint-positive
proportion above each original cutoff.

The frozen selection rule requires a weighted proportion of at least
95\% with nonzero support, choosing the least strict qualifying cutoff.
Only a selected cutoff proceeds to fresh validation, whose one-sided
95\% finite-population lower bound must reach 90\%. Neither version
selected a cutoff, so fresh validation was not run. In version 4, the
largest weighted proportion is 40\%, reached at $.99$ and $.995$
(Table~\ref{tab:calibration}). All 51 negative decisions were independently
recomputed. Appendix~\ref{app:protocol} gives invalid-response counts,
order diagnostics, sensitivity analysis, and the full decision protocol.

\begin{table}[t]
\centering
\small
\caption{Version 4 semantic selection at representative frozen cutoffs.
Pairs are sampled canonical pairs with score strictly above the cutoff.
Joint proportions use inclusion weights; the selection requirement is 95\%.}
\label{tab:calibration}
\begin{tabular}{rrr}
\hline
Cutoff & Pairs & Weighted joint positive (\%)\\
\hline
.90 & 330 & 0.01\\
.95 & 180 & 0.31\\
.98 & 90 & 9.66\\
.99 & 60 & 40.00\\
.995 & 30 & 40.00\\
\hline
\end{tabular}
\end{table}

The completed natural-data evidence establishes structural admissions and
exact membership repair on these development graphs. Both semantic
calibrations failed, leaving the threshold-qualified task, comparative
cost, and held-out effect study unexecuted. This failure concerns the fixed
representation, rater pair, and rubric; it does not establish the absence
of semantic duplicates in AskUbuntu. The original evaluation reservation
and unmet human-validation gates remain preserved.

\section{Related work}\label{sec:related}
\paragraph{Curation and complete-pipeline targets.}
\citet{cao} use summation-form learning to update feature selection and
models after deletion. \citet{ivmunlearning} casts unlearning as view
maintenance, including preprocessing in the maintained dataflow.
Dominici's thesis concerns data preparation pipelines and feature stores
\citep{dominici}; its metadata and research description were available
for comparison, while its full text was not. Our result fixes record-level
suppression and characterizes its exact statistic map under a decreasing
deletion budget.

SemDeDup studies embedding-space redundancy \citep{semdedup}; its released
suppression computation motivates our operator \citep{semcode}. ARCANE
combines representative selection with an exact-unlearning architecture
\citep{arcane}. Its Section~3.2 assigns no retraining cost to deletions
outside the selected subset. Our example shows why that property depends
on the selection contract. GRACE selects forget and retain coresets for
behavioral language-model unlearning \citep{grace}, a different target
from reconstructing a fixed curator and learner.

\paragraph{Provenance and additive repair.}
Polynomial provenance \citep{green}, provenance for aggregation and
difference \citep{aggregate}, DBSP's incremental query framework
\citep{dbsp}, and provenance-based regression updates \citep{priu}
provide the algebraic foundations. Our signed expansion couples substitution
with a decreasing degree cutoff. It preserves the fresh coefficient map,
including coefficients for future requests.

\paragraph{Continuing state and storage.}
\citet{ginart} include metadata for subsequent deletions in the object
compared with retraining. \citet{cohencontrol} relate unlearning to adaptive
history independence and distinguish current outputs from internal state
and later behavior. Learning-unlearning space complexity studies
budget-dependent storage \citep{space}. Ticketed unlearning separately
charges central information and per-record tickets \citep{ticketed},
including a multi-request threshold construction. Our bound charges all
private information accessible to repair, including tickets, under a
fixed public contract. Its matching codec concerns the constructed
feature family at each remaining horizon.

\paragraph{Information for counterfactual queries.}
\citet{counterfactual,counterfactualcode} study query-sufficient summaries
of binary histories. Their count-threshold example separates the state
needed for prediction from that needed for counterfactual answers.
Our construction ties the queries to deletion of public blockers and the
answers to exact ridge heads on the same finite features used by curation.
The approximate extension uses classical information-distortion tools
\citep{shannon} for expected squared head error on each fixed input and
probe. This criterion differs from the joint answer-tuple success event
in the counterfactual counting proof. Appendix~\ref{app:comparison}
records the precise comparison of these guarantees.

\section{Conclusion}
For exact repair after curation, the cumulative deletion budget determines
which excluded records can still matter. We turn this retention contract
into a repairable statistic summary. The finite-feature family shows that
it can require information absent from the current model. A service should
specify its curation target, cumulative horizon, and retained-data access
before discarding payloads.

\appendix
\section{Update accounting and state boundaries}\label{app:work}
The low-level integer-ID kernel rejects repeated, already deleted, invalid,
and over-budget IDs before mutation. The empirical service first
deduplicates IDs and removes already deleted IDs, so retries are no-ops,
including after budget exhaustion. It rejects invalid IDs and requests
with too many fresh IDs. The mathematical reference also treats retries as
no-ops. These interfaces apply the same substitution to valid fresh
deletions.

For singleton deletion of $u$, first obtain the inverse-incidence handles
of keys containing $u$. For each entry, remove its incidence at $u$, its
old dictionary key, and its degree-bucket membership. Remove $u$ from
the key. If the reduced key is absent, move the same entry handle to it;
all surviving inverse references remain attached to that handle and are
not rebuilt. If the key already exists, merge the values and remove all
remaining incidences of the retired source entry. A zero result also
retires the destination and its incidences. Moving a value transfers its
reference; only merges perform vector arithmetic and a zero test. After
reducing the horizon by one, remove the now-excessive degree bucket.
No key gains an element, and the number of live keys never increases.

Give each initial incidence a token. A shrink removes a token; a merge
or prune removes its source tokens. None are created. Thus at most $L_0$
inverse-incidence removals occur over the whole sequence.
Each merge decreases key count, giving at most $M_0$ vector additions.
A tuple shrink from size $j$ costs $O(j)$, and decreases its potential
$j(j+1)/2$ by $j$. This charges tuple copying, hashing, and comparisons
under the stated dictionary model. Streaming word-sized input IDs costs
$O(Q)$, including retries, for the abstract reference algorithm, which
does not sort requests. Per-call overhead adds $O(B)$ for $B$ calls; this
is absorbed into $O(Q)$ only for nonempty calls. The resulting core bound
is Equation~\eqref{eq:work}.

The frozen empirical service has additional work. Let $N$ be its fixed
public unit-ID count, and let call $b$ supply $q_b$ IDs with $f_b$ fresh
unique IDs, so $Q=\sum_b q_b$. It reconstructs the $N$-entry lookup on
every call and sorts the fresh IDs before invoking the kernel. Under
unit-cost identifiers and expected constant-time hashing, its update cost is
\[
 O\!\left(B(N+1)+Q+\sum_b f_b\log(f_b+1)+\Phi_0+sM_0\right).
\]
Direct low-level calls on arbitrary unsorted lists instead have the bound
\[
 O\!\left(B+Q+\sum_b q_b\log(q_b+1)+\Phi_0+sM_0\right).
\]
The service supplies a sorted list to the kernel, so a second arbitrary-order
sort is not charged. Its temporary lookup requires $O(N)$ words in addition
to request containers and the core state. String decoding, hashing and
comparison costs require separate byte accounting. Empty calls still cost
work. These API costs must be included in complete lifecycle measurements.

The coefficient values and incidence metadata occupy
$O(sM_0+M_0+L_0)$ words, plus membership storage. Rational arithmetic
requires additional bit-length accounting. A canonical export removes
zero keys and orders keys and values, with at least linear output cost.
The fast in-memory structure is canonical only as a logical map modulo
handles, not as allocator memory. A final exact map can be equal to a
fresh construction while both a floating-point accumulation order and
its serialized bytes differ. Numerical checks against a fresh oracle
must therefore state their contract rather than borrow exact equality.

\section{Finite-family graph and scorer proof}\label{app:scorer}
All statements in this appendix concern the mathematical rows of
Equation~\eqref{eq:rows}, not natural encoder outputs. Powers-of-four
$k,D$ make $1/\sqrt{k}$ and $1/\sqrt D$ dyadic. The dimension cap keeps
every nonzero stored value in the ordinary normal range; the range argument
below also excludes nonzero subnormal intermediates in scoring.
The norms, independent of private signs, are
\[
 \norm a^2=1,\quad \norm{s_j}^2=\norm{p_i}^2=65/64,
 \quad \norm{w_i}^2=16641/16384=(129/128)^2.
\]
The anchor blocks every $s_j$ and $p_i$ since its cosine with them is
$4/\sqrt{65}>.49$. It has zero dot product with every candidate.
Common-blocker/candidate cosine is
$448/(129\sqrt{65})>.43$; individual-blocker/own-candidate cosine is
$784/(129\sqrt{65})>.75$. For $i\ne j$, the individual-blocker/candidate
cosine is at most $784/(6\cdot129\sqrt{65})<.126$.
Candidate-candidate cosine is at most
\[
 \frac{1/4+49/(64\cdot6)+1/16384}{16641/16384}
                 =\frac{18563}{49923}<.372.
\]
These inequalities cover every earlier candidate blocker decision.
Additional edges among the public blockers do not affect the proof,
since the anchor initially suppresses them and they have zero labels.

The exact stored threshold is
\[
 \tau=\frac{3602879701896397}{2^{53}}
      =\frac25+\frac{1}{5\,2^{53}}.
\]
For this family, raw squared coordinates and partial norm-square sums
are exact in binary64: they are multiples of $2^{-34}$ and have
magnitude below two. Assume round-to-nearest ties-to-even, correctly
rounded square root and division, and the specified coordinate order,
without arithmetic-changing reassociation or clipping. With
$u=2^{-53}$, set
\[
 \rho=\frac{2u}{1-u},\qquad
 \gamma=\frac{(d+1)u}{1-(d+1)u}.
\]
The normalized coordinate relative error is at most $\rho$. Products
therefore contribute at most $2\rho+\rho^2$ in absolute cosine error,
using Cauchy--Schwarz. Ordered dot-product rounding contributes at most
$\gamma(1+\rho)^2$. Hence
\[
E_d=2\rho+\rho^2+\gamma(1+\rho)^2<2^{-32}
                    \qquad(d\le2^{20}).
\]
To justify the relative-rounding model even under cancellation, nonzero
computed normalized coordinates have magnitudes between $2^{-18}$ and
two. Rounded nonzero products are at least $2^{-37}$ and are multiples
of $2^{-89}$. Ordered sums remain on that grid: rounding either preserves
the value or moves to a coarser subgrid. A sum is therefore zero or has
magnitude at least $2^{-89}$, above the subnormal range. The crude bound
$4d\le2^{22}$ excludes overflow. These facts cover the norm, division,
product, and accumulation range assumptions used above.
This cannot cross the edge/nonedge margins. It proves the uniform graph
under the stated scorer, independently of the $mD$ private signs.
Native runtime admission and implementation replay remain separate checks.

Public signature existence is also qualified. Hadamard rows give the
stated orthogonal case. Alternatively, independent normalized signs have
pairwise tail probability at most $2e^{-k/72}$. A union bound gives a
code when $k\ge\lceil72\log(2m(m-1))\rceil$, choosing an admitted
power of four that still obeys the dimension cap. This does not permit
arbitrarily many signatures at fixed dimension.

\section{Squared-error information bound}\label{app:information}
Let $Y_{ij}$ be independent uniform bits and express the hidden signs as
$2Y_{ij}-1$. There are $n=mD$ bits. Given a state $M$ and independent
public coins $R$, rescale an approximate hard-probe head by
\[
 \widehat Y_{ij}=\frac12\left(1+
                       \frac{\kappa\sqrt D}{\alpha}\widehat\theta_i[H_j]\right).
\]
Clipping to $[0,1]$ cannot increase squared error, and
\[
 \frac1D\sum_j(\widehat Y_{ij}-Y_{ij})^2
 \le\frac{\kappa^2}{4\alpha^2}\norm{\widehat\theta_i-\theta_i}^2.
\]
Posterior means $p_l=\Pr(Y_l=1\mid M,R)$ minimize conditional squared
error, even if a decoder uses fresh randomness. Thus average posterior
variance is at most $v=\kappa^2\varepsilon/(4\alpha^2)$. For
$0\le s\le1/4$, the function
\[
 g(s)=h_2\!\left(\frac{1-\sqrt{1-4s}}2\right)
\]
is increasing and concave. One derivative check substitutes $s=q(1-q)$;
the needed sign follows from
$r-1/r-2\log r\ge0$ for $r\ge1$, whose derivative is
$(r-1)^2/r^2\ge0$ and value at one is zero. Since
$h_2(p)=g(p(1-p))$, conditional subadditivity and Jensen give
\[
 H(Y\mid M,R)\le\sum_l\mathbb E h_2(p_l)\le n g(v).
\]
The state alphabet contract implies $H(M\mid R)\le\bits$, so
$\bits\ge I(Y;M\mid R)\ge n[1-g(v)]$, proving
Equation~\eqref{eq:rate}. For $v\ge1/4$ the zero bound applies. Correctness
averaged over uniform signs and a uniform hard-probe index already
suffices. Averaging over all valid deletion sets is weaker because many
produce a zero head and can dilute the error.

The zero-information threshold is attained by returning just the public
component of the hard-probe head. Its squared error is $\alpha^2/\kappa^2$
for every sign assignment. The repository also gives a finite covering
code upper bound for a query summary. It can require an enormous public
codebook and does not give an efficient sequentially erasable lossy codec;
we make no such implementation claim here.

\section{Finite-output and utility limits}\label{app:certificate}
The hidden coordinate magnitude in a hard-probe exact head is
$a_H=\alpha/(\kappa\sqrt D)$. A deterministic certified parameter radius
strictly below $a_H$ identifies every hidden sign and transfers the
exact $mD$ private-bit bound. If $0<\lambda\le1$, then $\kappa<4$ and
$\sqrt D\le2^{10}$, giving $a_H>2^{-19}$. An actual radius at most
$2^{-20}$ is therefore sufficient. A configured tolerance is not a
certificate: the radius must cover moment accumulation, solving, and
release rounding relative to the prescribed exact target. We have not
established such a certificate for the ongoing natural-text learner.

For a public test distribution uniformly sampling an $H$-basis vector,
prediction discrepancy is $\norm{P_H e}^2/D$ for head error $e$.
Uniform sampling of all coordinate basis vectors instead yields
$\norm e^2/d$. Consequently a constant parameter-error obstruction does
not establish constant task risk. Native tag prediction, natural source
requests, and matched-cost baselines are needed to assess practical value.

\section{Reproducibility, compute, and assistance disclosure}\label{app:repro}
This standalone source embeds the unmodified official TMLR style and its
bibliography. The anonymous supplement contains an explicit subset of source
code, exact annotation prompts and schemas, and aggregate evidence.
The code includes data adapters, embedding and task configuration, development
selection, pair sampling, and the production-reused rendering and inference
kernels. Its source map distinguishes unchanged research modules from a
small extracted support helper. Deployment wrappers and private activation
locks are excluded; the package documents these remaining inputs.
A top-level \texttt{README.txt} maps claims to files, and
\texttt{python3 -B check\_all.py} runs the public checks with Python 3.10
or later, without model downloads or private data. The full
development archive retains historical proofs, reviews, code, negative results,
and frozen protocols. It is separate from the anonymous submission materials;
identifying infrastructure receipts and private observations are not included
in the supplement.

The active representation is \texttt{intfloat/multilingual-e5-base}.
The pinned model revisions are:\\
E5: \texttt{d128750597153bb5987e10b1c3493a34e5a4502a}.\\
Qwen: \texttt{a09a35458c702b33eeacc393d103063234e8bc28}.\\
Mistral: \texttt{04d8a90549d23fc6bd7f642064003592df51e9b3}.
Protocol manifests
bind tokenizer files, templates, prompt/schema, runtime, and sampling
seeds. Source/cache/task intake and independent arithmetic reviews bind
private outputs by exact bytes and hashes. Checksums support identity;
they do not establish semantic correctness or source rights.

\paragraph{Full-document embedding recipe.}\label{app:embedding}
Content is tokenized once without special tokens or truncation and split
into consecutive blocks of at most 256 content tokens. Separately tokenized
\texttt{query: } IDs are prepended to each block before the tokenizer adds
special tokens; content is never decoded and retokenized. For the pinned
E5 tokenizer, chunks are padded to 261 tokens. Batches contain chunks of
one document only. The encoder runs in evaluation mode on CPU in FP32,
with eager attention and deterministic algorithms enabled. Pooling masks
padding and averages all attended tokens, including prefix and special
tokens. Each FP32 mean is normalized using coordinate-ordered FP64 norm
accumulation. In document order, normalized chunk vectors are weighted by
their content-token counts and accumulated in FP64. Their weighted mean
is normalized again and cast to the stored FP32 row. Cosine scoring
renormalizes those stored rows in FP64. The supplement provides the exact
embedding and scoring source; the model card alone does not specify this
full-document procedure. Absolute scores and the failed calibration are
conditional on this fixed recipe, not a conclusion about every use of E5.

\paragraph{Asset attribution and replay limits.}
The pinned publisher cards declare MIT for
\href{https://huggingface.co/intfloat/multilingual-e5-base/blob/d128750597153bb5987e10b1c3493a34e5a4502a/README.md}{multilingual E5}
and Apache 2.0 for
\href{https://huggingface.co/Qwen/Qwen2.5-7B-Instruct/blob/a09a35458c702b33eeacc393d103063234e8bc28/README.md}{Qwen2.5-7B-Instruct}
and \href{https://huggingface.co/mistralai/Mistral-Nemo-Instruct-2407/blob/04d8a90549d23fc6bd7f642064003592df51e9b3/README.md}{Mistral-Nemo-Instruct-2407}.
We distribute no model weights. Stack Exchange's
\href{https://stackoverflow.com/help/licensing}{licensing guidance}
associates contributions from 2 May 2018 UTC onward with CC BY-SA 4.0
and identifies per-revision licenses through post timelines.
Our adapter retains supplied post attributes, including license and edit
metadata when present. This is not a complete author-history audit;
owner accounts remain contributor proxies. The public package excludes
corpus text, per-record identities and per-pair model observations.

The anonymous supplement checks its file hashes and the supplied aggregate
calibration and compute arithmetic. This verifies an extract of the evidence,
not the authenticity of the private observations. Recomputing the selection from individual observations
requires the private sample and audited model outputs; repeating inference
also requires the pinned source, assets and runtime. These are separate
verification levels. Fixed revisions and seeds alone do not establish
bitwise inference portability. The attribution ledger records documentary
sources and remaining gaps without changing the original acceptance gates.

\paragraph{Protocol and implementation history.}
The fully LLM branch does not relabel model judgments as human evidence or
inherit acceptance of the original human-gated study. The documentary
source chain is not publisher cryptographic attestation, and development
source acceptance is separate from earlier source gates.
Separate exact-byte audits verified prompt, token, and grammar identities
without repeating model calls. The prior Qwen audit's newline-normalization
failure and all original outputs remain preserved.
The algebraic, scorer, and codec artifacts retain their internal reviews;
these are implementation/reasoning checks, not external peer review.

\paragraph{Completed development preparation.}
Actual source/cache/task bindings and the source-disjoint development
allocation passed a separate implementation review. A resource pilot
scanned 262,144 real-cache cross-pairs and matched the original guard's
canonical first witness, score, and keep decision at $.50,.99,.995$.
The scan kernel took 0.7783 seconds in that pilot. This does not establish
full-scale runtime, full-guard completion, or empirical effect size.
The first stable-ID witness is not the maximum-scoring witness.

\paragraph{Observed compute scope.}
The two completed v4 inference jobs each allocated one A100 80GB GPU.
Their scheduler elapsed times were 11,836 seconds for Qwen and 7,209 for
Mistral, totaling 19,045 GPU-allocation seconds (5.29 hours). This includes
job overhead and model loading, excludes queue time, and is not measured
GPU utilization, energy, or end-to-end project time. The corrected Qwen
and Mistral CPU audits took 755 and 130 seconds; the retained failed Qwen
audit took 583 seconds, and selection statistics took one second.
These four CPU-job durations sum to 1,469 seconds, not CPU-core time.

A broader historical ledger covers 72 saved jobs, including those six.
Adding the two successful engineering pilots and the two v3 selection
jobs gives six model jobs totaling 25,759 A100-allocation seconds
(7.16 hours). The 32 E5 tasks, each encoding and then replaying its shard,
sum to 479,503 task-elapsed seconds (133.20 hours). Each requested two
CPUs, giving 266.39 requested-CPU hours when multiplied by elapsed time;
this is not measured CPU utilization or verified billed core time.
The remaining 34 stages sum to 8,396 job-elapsed seconds and include
preparation, acquisition, audits and a hardware probe. Four failed jobs
remain in the ledger. Parent jobs and array tasks are each counted once;
batch steps and raw/display aliases are not added again. Concurrent task
times overlap, so these sums are not project wall-clock duration.
Earlier setup, transfers and failed annotation attempts outside the saved
subset, local work, queue waits and energy remain unaccounted. A complete
project compute budget remains unavailable.

The subsequent structural diagnostic is separate from that historical ledger.
Its 51 graph constructions, alternate-tiling replays and checks took 248.513
seconds in the child process, with a reported peak resident set of
260,792,320 bytes. It requested two CPUs, used one numeric thread, and
allocated no GPU. These measurements reuse the existing feature cache;
they are neither full-pipeline cost nor a comparative efficiency result.

The separate repair diagnostic completed in 1,193.413 seconds, with
281,616,384 bytes of process peak RSS, reusing the existing embedding
cache. It requested two CPUs and no GPU, with one numerical thread.
These figures cover this diagnostic and its checks; they are not a
comparative service benchmark or the cost of the complete pipeline.

Literature coverage is also bounded. Primary publication pages and
official code were refreshed for the citations used here. The full
ARCANE selection description and the central, ticketed, and bounded-deletion
contracts in \citep{ticketed,space} were checked against their primary PDFs.
This is a contract comparison, not an independent audit of those proofs. The full six-page counterfactual-auditing manuscript \citep{counterfactual}
was inspected, including its sufficiency framework, Theorem~8 and proof,
validation scope, and limitations. Its approximation comparison above uses
the joint answer-tuple event required by that proof; the previously inspected
implementation remains separately cited \citep{counterfactualcode}.
Dominici's thesis \citep{dominici} remains unavailable in full text. The author
and laboratory pages establish its pipeline-unlearning scope, but do not
establish its operator-level guarantees or experimental results. That
comparison remains incomplete; absence of the text does not establish
absence of a corresponding result.
\paragraph{AI assistance.}
AI tools assisted manuscript drafting, code and figure preparation, source
checking, and internal review. Those internal checks are not external peer review, formal
verification, or authentication of private observations. The two annotation
models and their distinct empirical role are specified in
Section~\ref{sec:judges}. No AI system is listed as an author.

\section{Preserved downstream design: not executed}\label{sec:design}
Neither calibration failure was retuned away. The failed selection blocks
the full leakage guard, post-guard panels, threshold-qualified development
runs, and dependent held-out evaluation. The prepared guard package was
not launched. The original 21,332 prospective jobs are a design inventory,
not executed evidence. The original 100,000- and 200,000-row panels are
infeasible on this preparation; unavailable cells remain unfilled and are
not supplemented from held-out data. The original evaluation reservation
and unmet human-validation requirements remain unchanged.

The full development guard is conditional on passed fresh model
validation. It compares every reserved training row against all original
calibration and test rows, removing training leakage using the locked
threshold and the original scorer. This uses features and IDs, not later
evaluation labels. The intended scope is
$10{,}314\times42{,}063=433{,}837{,}782$ pairs. The 18,929 evaluation-test
labels stay uninspected during development. Development utility may use
only the 2,138 reserved test rows.

The planned requests include uniform record deletion, uniform genuine
source withdrawal, deletion among initially excluded records, and a
label-blind graph search for high-amplification stress requests. Uniform
record controls match source requests by removed-record volume. Stress
search does not estimate ordinary user behavior and is not an optimal
adversarial audit. Request streams, feasible panels, and resource bounds
must be sealed before their effect outcomes; the historical registry is
not authorization to mark every planned family as executed.

O-G recomputes the fixed curator and learner from surviving input; O-T
uses correct selected IDs but rebuilds moments. P-I uses indexed
coefficient repair, P-S scans coefficients, and P-R uses rank-coordinate
summaries with its reconstruction work charged. The essential B-E
baseline retains compact eligible feature/label rows, blocker incidences,
and buckets, then makes batched signed moment updates. It is not forced
to store an outer product for every row. B-A retains all payloads.
B-F deletes from the initial selected set and is a target-mismatch
diagnostic, not a correctness oracle. The smallest natural pilot must
compare at least P-I, O-G, and optimized B-E under identical requests.

Primary task quality is native multilabel mean squared error, with
micro/macro F1 and mean average precision secondary. Report selected-set
admissions, parameter and prediction deviations from the fresh oracle,
and the quality difference of the frozen-selection target. A large
admission count alone does not establish useful model change.
Costs include feature/cache access, graph/index state, moments, gathered
rows, peak workspace, serialization, initialization, updates, and solves.
Incremental timing and complete lifecycle timing are separate. A speedup
over fresh graph reconstruction alone is insufficient evidence of an
advantage over a strong incremental baseline.

\section{Limitations and broader impact}\label{sec:limits}
The theory assumes restriction consistency, fixed features, and a fixed
cumulative budget. It does not cover arbitrary reclustering, encoder
refitting, overlapping ownership, or nonconvex language-model fine-tuning.
The pinned E5 model card lists Stack Exchange among its pretraining
sources; overlap with the present corpus is unknown. Our target holds
that encoder fixed and does not remove influence from encoder pretraining.
The conditional retained-state bound fixes the deleted set $F$. An adaptive
request policy can reveal private information through the requests themselves.
The finite-family lower bound measures private information and exact
rational heads. It neither proves total-byte optimality nor shows that
natural text embeddings contain the constructed private sign patterns.
A floating solver needs an actual error certificate before an exact-head
lower bound transfers to its releases (Appendix~\ref{app:certificate}).

The LLM-only study can establish properties of its model-defined labels.
Shared training data and correlated biases can make both judges wrong;
agreement is not human validation. Invalid responses remain part of the
criterion, and order sensitivity is reported separately. Existing native
tags are dataset labels, not newly obtained human annotations. No human pair review was used, and no human-ground-truth claim is
available. This annotation boundary is separate from the authors' responsibility
for the paper and its conclusions.

The study uses an existing public forum export and recruits no new
participants. Native account identifiers are contributor proxies. Export filters and
unknown ownership limit the scope of source withdrawal. Source text may
contain personal information, so bodies and account-level manifests stay
private. Redistribution and attribution requirements must be resolved
before any raw-data release. Canonical state equality does not imply
physical erasure, privacy of prior answers, or compliance with every
possible deletion obligation. A useful application is making deletion-service
contracts explicit. A corresponding risk is retaining excluded personal data
in the name of future repair or mistaking model agreement for permission to
remove legitimate content. Data minimization, access controls, and a clearly
limited deletion horizon remain necessary operational safeguards.

The natural-text effect study was not executed. Accordingly, this paper
makes no claim of natural semantic value or competitive efficiency. Both
failed calibrations and the earlier unfavorable memory comparison remain
part of the evidence. The theoretical results identify an exact target and
its information requirements; they do not validate a deployment.

\section{Detailed comparison of information contracts}\label{app:comparison}
Learning-unlearning space complexity \citep{space} gives budget-dependent
upper and lower bounds for realizability testing in Theorems~5.4 and~6.3
of the full version. Appendix~G connects that task to empirical risk
minimization. Ticketed schemes \citep{ticketed} separate central auxiliary
information from tickets provided with forgotten records. Their main
Definitions~1--2 and Remark~3 specify exact predictor recovery after one
request; the threshold construction in Appendix~B.1 supports multiple
requests. Our private-information accounting includes all such usable
side information.

For binary histories of length $t\ge3$, Theorem~8(i--ii) of
\citet{counterfactual} requires $t+1$ states for all count-threshold
predictions and $2^t$ states for all single-bit-flip threshold queries.
The exact reconstruction argument and released implementation
\citep{counterfactualcode} concern query-sufficient history summaries.
Theorem~8(iii)'s counting proof requires at least $(1-\delta)2^t$ states
when the entire answer tuple is correct with probability at least
$1-\delta$ under uniform histories. This is a joint success event, not a
separate marginal success guarantee for each query. Our approximate
criterion bounds expected squared ridge-head error for each fixed input
and designated deletion probe. These numerical bounds have different
error contracts and are not interchangeable.

\section{Calibration protocol and prepared task}\label{app:protocol}
\subsection{Population, sources, and reserved task configuration}
The 140,250 prepared questions split into 98,187 training, 20,996
calibration, and 21,067 test rows. The source grouping contains 88,344
units, including 1,522 unknown units. Missing and community owners are
explicit unknown singletons. Source completeness concerns retained records
after the fixed export and filters, not every historical or private
contribution. Acquisition is checksummed and the export provenance is
recorded. Bodies and account-level records remain private.

A deterministic one-tenth whole-source reservation allocated 10,314
training and 2,138 test rows to development. Evaluation retains 87,873
training and 18,929 test rows; calibration is unchanged. Every reserved
source stays outside later evaluation, even when a pilot uses only part
of the allocation. Planned development targets of 1,000, 2,500, and 5,000
training rows include the entire boundary source and use frozen priority
after the guard. These are upper bounds before leakage filtering. The
executed 1,001-row diagnostic uses the original pre-guard prefix, as stated
in Section~\ref{sec:study}.

The prepared prediction task uses 20 native tags and retains all-zero
target vectors. Five source-grouped calibration folds selected
$\lambda=10^{-5}$ from $10^{-6},10^{-5},\ldots,10^{-1}$, with larger
$\lambda$ on exact ties. Out-of-fold predictions fix one F1 threshold per
tag, choosing the stricter threshold on ties and a never-positive rule
when no calibration positive exists. Test labels and deletion outcomes
are not used for configuration. Different representations need separate
task settings. These preparatory choices are not a completed downstream
deletion experiment.

\subsection{Inference, sampling, and validation rules}
The models run offline on admitted A100 GPUs with BF16, batch size one,
greedy constrained decoding, forced Flash attention, a 512-token output
budget, and a 32,768-token total budget. Truncation, summarization, format
repairs, semantic retries, and fallback attention are disallowed. Pair
texts are untrusted inputs; judges have no tools or retrieval. The
supplement supplies the exact instructions, schemas, renderer, and
constrained-inference kernels. No independence between the two raters is
assumed.

Version 4 includes same-source pairs in its calibration-pair population.
Its 13 right-closed score strata use cuts
$.50,.90,.91,\ldots,.99,.995$. Each nonempty stratum contributes a simple
random sample without replacement of up to 30 pairs. The 360 canonical
pairs support the weighted selection calculation; the 36 swapped pairs
are diagnostics and supply no extra votes.

A selected cutoff would proceed to a fresh simple random sample of up to
200 pairs from all above-threshold calibration pairs. Its exact one-sided
95\% finite-population lower bound must reach $.90$. Any change in a
model's availability/status or valid category on canonical overlap between
selection and validation makes validation unavailable. This bound concerns
the fixed/potential model-label function under the sampling model. There
is no retuning after failure and no fresh-validation result for either
completed selection version.

\subsection{Response and sensitivity diagnostics}
Version 3 attempted all 198 model calls. One canonical Mistral output was
invalid JSON and remained in the denominator without retry. All 51 cuts
were recomputed independently. Version 4 completed 396 calls per model.
Qwen returned 391 valid outputs and five invalid JSON outputs; Mistral
returned 396 valid outputs.

The version 4 stratified sample contains 27 joint positives among 360
canonical pairs. This unweighted ratio is not a population prevalence
estimate. Qwen marks 32 pairs positive and Mistral 151. Their six-way
categories differ on 308 of the 355 pairs with two valid outputs. Swapping
pair order changes status or category in three Qwen and 14 Mistral cases;
the binary-positive decisions change in one and six cases, respectively.
These characterize the raters, not accuracy against human judgments.

The frozen favorable-imputation sensitivity treats invalid or uncertain
judgments as potentially positive while preserving substantive negatives.
At cuts $.90,.95,.98,.99,.995$, the weighted values are respectively
$1.72\%,5.64\%,10.08\%,44.68\%,53.33\%$. The maximum over all cuts is
$8/15$, still below the selection requirement. These are sensitivity
values, not confidence limits. The all-cut ledger preserves exact
rational arithmetic and the original negative decisions.

\end{document}